%% file: root.tex
\documentclass{article}
\usepackage{spconf}

\usepackage[english]{babel}

\usepackage{ifpdf}

\usepackage{cite} 
\usepackage{url}
\usepackage{hyperref}

\usepackage[pdftex]{graphicx}
\graphicspath{{./}}
\usepackage{color}
\usepackage{pgf, tikz, pgfplots}
\usetikzlibrary{shapes, arrows, automata, plotmarks}
\usetikzlibrary{calc,hobby,decorations}

\usepackage[cmex10]{amsmath}
\usepackage{amsfonts, amssymb, amsthm}
\usepackage{mathrsfs}

\usepackage{algorithm,algpseudocode}
\algnewcommand{\LeftComment}[1]{\Statex \(\triangleright\) #1}

\usepackage[margin=15mm]{geometry}
\usepackage{enumerate}
\usepackage{multirow}
\usepackage{rotating}
\usepackage{subcaption}
\input{pennColors.sty}
\input{mySections.tex}

\input{mySymbol.sty}
\input{myTikzDefs.tex}

\newtheorem{assumption}{\hspace{0pt}\bf Assumption}

\newtheorem{theorem}{\hspace{0pt}\bf Theorem}

\title{WIDE AND DEEP GRAPH NEURAL NETWORKS WITH DISTRIBUTED ONLINE LEARNING}
\twoauthors
{Zhan Gao, Alejandro Ribeiro\thanks{Supported by NSF CCF 1717120, ARO W911NF1710438, ARL DCIST CRA W911NF-17-2-0181, ISTC-WAS and Intel DevCloud.}}
{University of Pennsylvania\\
    Dept. of Electrical and Systems Engineering\\
    Philadelphia, PA}
{Fer\hspace{0.015cm}nando Gama}
{University of California, Berkeley\\
    Electrical Eng. and Computer Sci. Dept.\\
    Berkeley, CA}
\begin{document}
\ninept
\maketitle

\begin{abstract}
Graph neural networks (GNNs) learn representations from network data with naturally distributed architectures, rendering them well-suited candidates for decentralized learning. Oftentimes, this decentralized graph support changes with time due to link failures or topology variations. These changes create a mismatch between the graphs on which GNNs were trained and the ones on which they are tested. Online learning can be used to retrain GNNs at testing time, overcoming this issue. However, most online algorithms are centralized and work on convex problems (which GNNs rarely lead to). This paper proposes the Wide and Deep GNN (WD-GNN), a novel architecture that can be easily updated with distributed online learning mechanisms. The WD-GNN comprises two components: the wide part is a bank of linear graph filters and the deep part is a GNN. At training time, the joint architecture learns a nonlinear representation from data. At testing time, the deep part (nonlinear) is left unchanged, while the wide part is retrained online, leading to a convex problem. We derive convergence guarantees for this online retraining procedure and further propose a decentralized alternative. Experiments on the robot swarm control for flocking corroborate theory and show potential of the proposed architecture for distributed online learning.
\end{abstract}



\section{Introduction} \label{sec:intro}

Graph neural networks (GNNs) \cite{Bruna2013, Defferrard2016, Fernando2019, Wu19-SGC, Kipf2017, Xu19-GIN} are nonlinear representation maps that have been shown to perform successfully on graph data in a wide array of tasks involving citation networks \cite{Kipf2017}, recommendation systems \cite{Ying2018}, source localization \cite{gao2020stochastic}, wireless communications \cite{gao2020resource} and robot swarms \cite{tolstaya2019learning}. GNNs consist of a cascade of layers, each of which applies a graph convolution (a graph filter) \cite{Ortega18-GSP}, followed by a pointwise nonlinearity \cite{Bruna2013, Defferrard2016, Kipf2017, Fernando2019, Xu19-GIN, Wu19-SGC}. One of the key aspects of GNNs is that they are local and distributed. They are local since they require information only from neighboring nodes, and distributed since each node can compute its own output, without need for a centralized unit.

GNNs have been developed as a nonlinear representation map that is capable of leveraging graph structure present in data \cite{Gama20-GNNs}. The most popular model for GNNs is the one involving graph convolutions (formally known as graph FIR filters). Several implementations of this model were proposed, including \cite{Bruna2013} which computes the graph convolution in the spectral domain, \cite{Defferrard2016} which uses a Chebyshev polynomial implementation, \cite{Fernando2019, Wu19-SGC} which use a summation polynomial, and \cite{Kipf2017, Xu19-GIN} which reduce the polynomial to just the first order. All of these are different implementations of the same representation space given by the use of graph convolutional filters to regularize the linear transform of a neural network model. Other popular GNN models include graph attention transforms \cite{Velickovic18-GraphAttentionNetworks} and CayleyNets \cite{Levie17-CayleyNets}; see \cite{Isufi20-EdgeNets} for a general framework.

Oftentimes, however, problems of interest exhibit (slight) changes in data structure between training and testing sets or involve dynamic systems \cite{frahling2008sampling, helwa2017multi, Gama19-Control}. For example, in the case of the robot swarm, the graph is determined by the communication network between robots which is, in turn, determined by their physical proximity. Thus, if robots move, the communication links will change, and the graph support will change as well. Therefore, we oftentimes need to adapt to (slightly) new data structures. GNNs have been shown to be resilient to changes, as proven by the properties of permutation equivariance and stability \cite{ZouLerman19-Scattering, Gama19-Stability}. While these properties guarantee transference, we can further improve the performance by leveraging online learning approaches.

Online learning is a well-established paradigm that tracks the optimizer of time-varying optimization problems and has been successful as an enabler in the fields of machine learning and signal processing \cite{dabbagh2005online}. In a nutshell, online algorithms tackle each modified time instance of the optimization problem, by performing a series of updates on the previously obtained solutions. In order to leverage online learning in GNNs we face two major roadblocks. First, optimality bounds and convergence guarantees are given only for convex problems \cite{shalev2012online}. Second, online optimization algorithms assume a centralized setting. The latter one is particularly problematic since it violates the local and distributed nature of GNNs, upon which much of its success has been built \cite{tolstaya2019learning}.

Online learning has been investigated in designing neural networks (NNs) for dynamically varying problems. Specifically, \cite{Sanz2012, Li2004} develop online algorithms for feedforward neural networks with applications in dynamical condition monitoring and aircraft control. More recently, online learning has been used in convolutional neural networks for visual tracking, detection and classification \cite{Hong2015, Molchanov2016}. While these works develop online algorithms for NNs, analysis on the convergence of these algorithms is not presented, except for \cite{Ho20102} that proves the convergence of certain online algorithms for radial neural networks only.

This paper puts forth the Wide and Deep Graph Neural Network (WD-GNN) architecture, that is amenable to distributed online learning, while keeping convergence guarantees. In analogy to \cite{cheng2016wide}, we define the WD-GNN as consisting of two components, a deep part which is a nonlinear GNN and a wide part which is a bank of linear graph filters (Section \ref{WD_GNN}). We propose to have an offline phase of training, that need not be distributed, and then an online \emph{retraining} phase, where only the wide part is adapted to the new problem settings. In this way, we learn a nonlinear representation, that can still be adapted online without sacrificing the convex nature of the problem (Section \ref{DOL}). We further develop an algorithm for \emph{distributed} online learning. We prove convergence guarantees for the proposed online learning procedure (Section \ref{PA}). Finally, we perform simulated experiments on the robot swarm control (Section \ref{E}). We note that proofs, implementation details and another experiment involving movie recommendation systems, can be found in appendices.


\section{Wide and Deep Graph Neural Network}
\label{WD_GNN}

Let $\ccalG = \{\ccalV, \ccalE, \ccalW\}$ describe a \emph{graph}, where $\ccalV = \{n_{1},\ldots,n_{N}\}$ is the set of $N$ nodes, $\ccalE \subseteq \ccalV \times \ccalV$ is the set of edges, and $\ccalW : \ccalE \to \reals$ is the edge weight function. In the case of the robot swarm, each node $n_{i} \in \ccalV$ represents a robot, each edge $(n_{j}, n_{i}) \in \ccalE$ represents the communication link between robot $n_{j}$ and $n_{i}$, and the weight $\ccalW(n_{j}, n_{i}) = w_{ji} \geq 0$ models the communication channel. The data of interest is defined on top of the graph and is described by means of a \emph{graph signal} $\bbX : \ccalV \to \reals^{F}$ which assigns an $F$-dimensional \emph{feature vector} to each node. For example, in the robot swarm, the signal $\bbX(n_{i})$ represents the \emph{state} of robot $n_{i}$, typically described by its position, velocity or acceleration. The collection of features across all nodes in the graph can be denoted with a matrix $\bbX \in \reals^{N \times F}$ which we call a graph signal as well. Note that each row of $\bbX$ corresponds to the feature vector at each node, whereas each column corresponds to the collection of $f$th feature across all nodes.

We leverage the framework of graph signal processing (GSP) as the mathematical foundation to learn from graph signals \cite{Ortega18-GSP}. Note that $\bbX$ is a $N \times F$ matrix that bares no information about the underlying graph (beyond the fact that it has $N$ nodes). To relate the graph signal to the specific graph it is supported on, we define the support matrix $\bbS \in \reals^{N \times N}$ that satisfies $[\bbS]_{ij} = 0$ if $i \neq j$ and $(n_{j},n_{i}) \notin \ccalE$, capturing the graph structure. Common examples in the literature include the adjacency matrix, the Laplacian matrix, and other normalized versions. The key aspect of the support matrix is that it respects the sparsity of the graph. Thus, when using it as a linear operator on the data $\bbS \bbX$, we observe that the output at node $n_{i}$ for feature $f$ becomes
\begin{equation} \label{eqn:graphShift}
    [\bbS \bbX]_{if} = \sum_{j=1}^{N} [\bbS]_{ij} [\bbX]_{jf} = \sum_{j : n_{j} \in \ccalN_{i}} [\bbS]_{ij} [\bbX]_{jf}
\end{equation}
where the second equality emphasizes the sparse nature of $\bbS$. That is, only the values of the $f$th feature at the neighboring nodes $n_{j} \in \ccalN_{i}$ for $\ccalN_{i} = \{n_{j} \in \ccalV: (n_{j},n_{i}) \in \ccalE\}$ are required to compute the output $\bbS \bbX$ at each node. This renders $\bbS \bbX$ a linear operation that only needs information from direct neighbors (local) and that can be computed separately at each node (distributed). The operation $\bbS \bbX$ is at the core of GSP since it effectively relates the graph signal with the graph support, referred to as a \emph{graph shift} \cite{Ortega18-GSP}.

While, in general, we can think of graph data as given by a pair $(\bbX,\bbS)$ consisting of the graph signal $\bbX$ and its support $\bbS$, we only regard $\bbX$ as \emph{actionable} and the support $\bbS$ is determined by physical constraints of the problem. Motivated by the previous study for recommendation systems \cite{cheng2016wide}, we propose the Wide and Deep Graph Neural Network (WD-GNN) architecture. It is a nonlinear map $\bbPsi:\reals^{N \times F} \to \reals^{N \times G}$ comprising two components
\begin{equation} \label{eqn:WDGNN}
    \bbPsi(\bbX;\bbS, \ccalA, \ccalB) = \alpha_{\text{D}} \bbPhi(\bbX; \bbS, \ccalA) + \alpha_{\text{W}} \bbB(\bbX; \bbS, \ccalB) + \beta
\end{equation}
where $\bbPhi(\bbX;\bbS, \ccalA)$ is called the \emph{deep part} and is a graph neural network (GNN), and $\bbB(\bbX; \bbS,\ccalB)$ is called the \emph{wide part} and is a bank of graph filters. The scalars $\alpha_{\text{D}}$, $\alpha_{\text{W}}$ and $\beta$ are combination weights\footnote{The combination weights $\alpha_{\text{D}}$, $\alpha_{\text{W}}$ and $\beta$ can also be considered as architecture parameters and be trained, if necessary.}. 

\subsection{Wide Component: Bank of Graph Filters}

The wide component is a \emph{bank of graph filters} \cite{Ortega18-GSP}. It is defined as a linear mapping between graph signals $\bbB: \reals^{N \times F} \to \reals^{N \times G}$, characterized by a set of weights or \emph{filter taps} $\ccalB = \{\bbB_{k} \in \reals^{F \times G}\}_{k=0}^K$ as follows
\begin{equation} \label{eqn:graphConv}
    \bbB(\bbX; \bbS, \ccalB) = \sum_{k = 0}^{K} \bbS^{k} \bbX \bbB_{k}.
\end{equation}
where the output is another graph signal of dimensions $N \times G$. The operation in \eqref{eqn:graphConv} is often called a \emph{graph convolution} \cite{Bruna2013, Fernando2019}. With the multi-dimensional nature of the feature vector, the graph convolution \eqref{eqn:graphConv} acts analogously to the application of a bank of $FG$ filters, hence the name. Oftentimes, though, we refer to \eqref{eqn:graphConv} as simply a \emph{graph filter} or a \emph{graph convolution} for convenience.

The filtering operation as in \eqref{eqn:graphConv} is distributed and local, as it can be computed at each node with information relied by neighboring nodes only. To see this, note that the multiplication on the left by $\bbS^{k}$ mixes information from different nodes (in the columns of $\bbX$). However, since $\bbS$ respects the sparsity of the graph, $\bbS \bbX$ is distributed [cf. \eqref{eqn:graphShift}] and $\bbS^{k} \bbX = \bbS(\bbS^{k-1}\bbX)$ can be seen as $k$ repeated applications of $\bbS$, then only exchanging information with direct neighbors $k$ times is required. Multiplication on the right by $\bbB_{k}$ carries out a linear combination of entries in the row of $\bbX$, but since each row is a collection of $F$ features at individual node, then they do not involve any exchange of information with neighboring nodes. Therefore, computing the output of a graph filter \eqref{eqn:graphConv} is carried out entirely in a distributed manner (each node computes its own output feature vector) and involves only local information (there is no need to know the structure nor the global information of entire graph).

In practice, the nodes do not need to know $\bbS$ at implementation time. They only need to have communication capabilities to receive the information from neighboring nodes, and computational capabilities to compute a linear combination of the information received from the neighbors. They \emph{do not require} full knowledge of the graph, but only of their immediate neighbors. Thus, the distributed implementation scales seamlessly to the graph filtering operation \cite{Gama19-Stability}. We fundamentally use \eqref{eqn:graphConv} as a mathematical framework that offers a condensed description of the communication exchanges that happen in a network.

\subsection{Deep Component: Graph Neural Network}

The deep component is a convolutional \emph{graph neural network} (GNN). It is defined as a nonlinear mapping between graph signals $\bbPhi: \reals^{N \times F} \to \reals^{N \times G}$, built as a cascade of graph filters [cf. \eqref{eqn:graphConv}] and pointwise nonlinearities
\begin{equation} \label{eqn:GCNN}
    \bbPhi(\bbX; \bbS, \ccalA) = \bbX_{L} \quad \text{with} \quad \bbX_{\ell} = \sigma \bigg( \sum_{k=0}^{K} \bbS^{k} \bbX_{\ell-1} \bbA_{\ell k} \bigg)
\end{equation}
for $\ell = 1,\ldots, L$, with $\sigma:\reals \to \reals$ a pointwise nonlinear function, which, for ease of exposition, denotes its entrywise application in \eqref{eqn:GCNN}; and characterized by the set of filter taps $\ccalA = \{ \bbA_{\ell k} \in \reals^{F_{\ell-1} \times F_{\ell}} \ , \ k = 0,\ldots, K, \ell=1,\ldots,L\}$. The graph signal $\bbX_{\ell}$ at each layer has $F_{\ell}$ features, and the input is $\bbX_{0} = \bbX$ so that $F_{0} = F$ while the output has $F_{L}=G$ features. The GNN is distributed and local inherited immediately from the graph filter since the pointwise nonlinearity does not affect the distributed implementation. 


\section{Distributed online learning} \label{DOL}

We train the WD-GNN \eqref{eqn:WDGNN} by solving the empirical risk minimization (ERM) problem\footnote{We took the license to define the ERM problem as in \eqref{eqn:ERM} so as to include supervised and unsupervised problems in a single framework. To use \eqref{eqn:ERM} for a supervised problem, we just extend $J$ to operate on an extra input representing the label given in the training set.} for some cost function $J: \reals^{N \times G} \to \reals$ in a given training set $\ccalT = \{\bbX_{1},\ldots,\bbX_{|\ccalT|}\}$
\begin{equation} \label{eqn:ERM}
    \min_{\ccalA, \ccalB} \frac{1}{|\ccalT|}\sum_{\bbX \in \ccalT} J \big( \bbPsi(\bbX;\bbS, \ccalA, \ccalB) \big).
\end{equation}
The ERM on a nonlinear neural network model is typically nonconvex. Training is the procedure of applying some SGD-based optimization algorithm for some number of iterations or training steps to arrive at some set of parameters $\ccalA^{\dag}$ and $\ccalB^{\dag}$. Note that we train the WD-GNN jointly by optimizing parameters of wide and deep components simultaneously. Also note that the number of parameters in $\ccalA$ and $\ccalB$ is determined by the hyperparameters $L$ (number of layers), $K_{\ell}$ (number of filter taps per layer) and $F_{\ell}$ (number of features per layer) for the deep part, and $K$ (number of filter taps) for the wide part.

In many problems of interest, the data structures may change from the training phase to the testing phase, or we may consider dynamic systems where the scenario changes naturally with time. The problem of controlling a robot swarm, for instance, exhibits both since different initializations of positions and velocities of the swarm lead to different structures between training and testing, and since inevitable movements of robots cause the communications links between them to change. \emph{Online learning} addresses this problem by proposing optimization algorithms that adapt to a continuously changing problem \cite{dabbagh2005online}. It operates by adjusting parameters repetitively for each time instance of the problem. However, online algorithms require convexity of the ERM problem to provide optimality bounds as well as convergence guarantees, and the ERM problem \eqref{eqn:ERM} using the WD-GNN is rarely convex.

To tackle this issue we propose to only \emph{retrain} the \emph{wide} component of the WD-GNN. By considering the deep part fixed $\ccalA = \ccalA^{\dag}$ as obtained from solving \eqref{eqn:ERM} over the training set, we focus on the wide part, which is linear. We then obtain a new ERM problem, now convex, that can be solved online to find a new set of parameters $\ccalB$. In essence, we leverage the deep part to learn a \emph{nonlinear representation} from the training set in the offline phase, and then adapt it online to the testing set, but only up to the extent of linear transforms.

Let $\ccalA^{\dag}$ and $\ccalB^{\dag}$ be the parameters learned from the offline phase. At testing time, the implementation scenario may differ from the one that is used for training, leading to a time-varying optimization problem of the form
 \begin{equation}\label{eq:timevaryingp}
\min_{\ccalA,\ccalB} J_t\big(\bbPsi(\bbX_t; \bbS_t, \ccalA, \ccalB)\big)
\end{equation}
where $J_t(\cdot)$, $\bbX_t$ and $\bbS_t$ are  the loss function, the observed signal and the graph structure at time $t$, respectively. In the online phase we fix the deep part $\ccalA = \ccalA^{\dag}$, converting the WD-GNN to a linear convex model. Then, we retrain the wide part online based on the changing scenario [cf. \eqref{eq:timevaryingp}]. More specifically, we let $\ccalB_0 = \ccalB^{\dag}$ initially and, at time $t$, we have parameters $\ccalA^{\dag}$ and $\ccalB_t$, input signal $\bbX_t$, output $\bbPsi(\bbX_t; \bbS_t, \ccalA^{\dag}, \ccalB_t)$, and loss $J_t\big(\bbPsi(\bbX_t; \bbS_t, \ccalA^{\dag}, \ccalB_t)\big)$. We then perform a few (probably one) gradient descent steps with step size $\gamma_t$ to update $\ccalB_t$
 \begin{equation}\label{eq:onlinelearning}
\ccalB_{t+1} = \ccalB_{t}-\gamma_t \nabla_{\ccalB} J_t\big(\bbPsi(\bbX_t; \bbS_t, \ccalA^{\dag}, \ccalB_t)\big).
\end{equation}

\noindent One major drawback in the above online learning procedure is that it is centralized, violating the distributed nature of the WD-GNN. To overcome this issue, we introduce the following distributed online algorithm.

\subsection{Distributed Online Algorithm}

In decentralized problems, each node $n_i$ only has access to a local loss $J_{i,t}\big(\bbPsi(\bbX_t; \bbS_t, \ccalA_i, \ccalB_i)\big)$ with local parameters $\ccalA_i$ and $\ccalB_i$. The goal is to coordinate nodes to minimize the sum of local costs $\sum_{i=1}^N J_{i,t}\big(\bbPsi(\bbX_t; \bbS_t, \ccalA_i, \ccalB_i)\big)$ while keeping local parameters equal to each other, i.e., $\ccalA_i = \ccalA$ and $\ccalB_i = \ccalB$ for all $i=1,\ldots,N$. We can then recast problem \eqref{eq:timevaryingp} as a constrained optimization one
  \begin{align}\label{eq:distvp}
&\min_{\{\bbB_i\}_{i=1}^N} \sum_{i=1}^N J_{i,t}\big(\bbPsi(\bbX_t; \bbS_t, \ccalA^{\dag}, \ccalB_i)\big),\\
&\quad {\rm s.t.}\quad \ccalB_i=\ccalB_j~\forall~i,j\!:\!n_j \in \ccalN_i\nonumber
\end{align}
where $\ccalA_i = \ccalA^{\dag}$ for all $i=1,\ldots,N$ since the deep part is fixed. The constraint $\ccalB_i=\ccalB_j$ for all $i,j:n_j \in \ccalN_i$ indicates that $\ccalB_i=\ccalB$ for all $i=1,\ldots,N$ under the assumed connectivity of the graph. To solve \eqref{eq:distvp}, at time $t$, each node $n_i$ updates its local parameters $\ccalB_i$ by the recursion
\begin{equation}
\begin{aligned}\label{eq:disgd}
\ccalB_{i,t+1} &= \frac{1}{|\ccalN_i|+1} \Big( \sum_{j:n_j \in \ccalN_i} \ccalB_{j,t} + \ccalB_{i,t}\Big) \\
    & \qquad - \gamma_t \nabla_\ccalB J_{i,t}\big(\bbPsi(\bbX_t; \bbS_t, \ccalA^{\dag}, \ccalB_{i,t})\big)
\end{aligned}
\end{equation}
with $|\ccalN_i|$ the number of neighbors of node $n_i$. Put simply, each node $n_i$ descends its parameters $\ccalB_{i,t}$ along the local gradient $\nabla_\ccalB J_{i,t}$ $\big(\bbPsi(\bbX_t; \bbS_t, \ccalA^{\dag}, \ccalB_{i,t})\big)$ to approach the optimal parameters of \eqref{eq:distvp}, while performing the average over the 1-hop neighborhood in the meantime to drive local parameters to the consensus. This online algorithm is decentralized that can be carried out locally at each node by only communicating with its neighbors.

In a nutshell, the proposed online learning procedure is of low complexity due to the linearity and guarantees efficient convergence due to the convexity as proved next. Furthermore, it can be implemented in a distributed manner requiring only neighborhood information.


\section{Convergence} \label{PA}

We show the efficiency of proposed online learning procedure by analyzing its convergence property. To be more precise, we establish that it converges to the optimizer set of a time-varying problem, up to an error neighborhood that depends on problem variations. Before claiming the main result, we need following standard assumptions.

\begin{assumption} \label{asm:1}
Let $J_t\big(\bbPsi(\bbX_t; \bbS_t, \ccalA, \ccalB)\big)$ be the time-varying loss function with fixed parameters $\ccalA$, and $\ccalB^*_t$ be an optimal solution of $J_t\big(\bbPsi(\bbX_t; \bbS_t, \ccalA, \ccalB)\big)$ at time $t$. There exist a sequence $\{ \ccalB_t^* \}_t$ and a constant $C_B$ such that for all $t$, it holds that
 \begin{equation}\label{eq:H}
\| \ccalB_{t+1}^* - \ccalB_t^* \| \le C_B.
\end{equation}
\end{assumption}

\begin{assumption} \label{asm:2}
Let $J_t\big(\bbPsi(\bbX_t; \bbS_t, \ccalA, \ccalB)\big)$ be the time-varying loss function with fixed parameters $\ccalA$. If $\bbPsi(\bbX_t; \bbS_t, \ccalA, \ccalB)$ is a linear function of $\ccalB$, $J_t\big(\bbPsi(\bbX_t; \bbS_t, \ccalA, \ccalB)\big)$ is differentiable, strongly smooth with constant $C_{t,s}$ and strongly convex with constant $C_{t,\ell}$.
\end{assumption}

Assumption \ref{asm:1} establishes the correlation between time-varying problems and bounds time variations of changing optimal solutions. Assumption \ref{asm:2} is commonly used in optimization theory and satisfied in practice \cite{boyd2004convex}. Both assumptions are mild, based on which we present the convergence result.

\begin{theorem} \label{thm:convergence}
Consider the WD-GNN \eqref{eqn:WDGNN} optimized with the propose online learning procedure in \eqref{eq:onlinelearning}. Let $J_t\big(\bbPsi(\bbX_t; \bbS_t, \ccalA, \ccalB)\big)$ be the time-varying loss function satisfying Assumptions \ref{asm:1}-\ref{asm:2} with constants $C_B$, $C_{t,s}$ and $C_{t,\ell}$. Let also $\ccalB^*_t$ be the optimal solution of $J_t\big(\bbPsi(\bbX_t; \bbS_t, \ccalA, \ccalB)\big)$ and $\gamma_t = \gamma$ be the step-size of gradient descent. Then, the sequence $\{ \ccalB_t \}_t$ satisfies
 \begin{equation}\label{eq:H}
\| \ccalB_t - \ccalB_t^* \| \le \big( \prod_{\tau=1}^t m_{\tau} \big) \| \ccalB_0 - \ccalB_0^* \| + \frac{1-\hat{m}^{t}}{1-\hat{m}} C_B
\end{equation}
where $m_t=\max\{ |1-\gamma C_{t,s}|, |1-\gamma C_{t,\ell}| \}$ is the convergence rate and $\hat{m}= \max_{1 \le \tau \le t} m_\tau$.
\end{theorem}

\begin{proof}
See Appendix \ref{Appendix:proofOfTheorem1}.
\end{proof}

Theorem \ref{thm:convergence} shows the online learning of the WD-GNN converges to the optimal solution of the time-varying problem up to a limiting error neighborhood. The latter depends on time variations of the optimization problem. If we particularize $C_B=0$ with $m_t = m$ for all $t$, we re-obtain the same result for the time-invariant optimization problem, that is the exact convergence of gradient descent.


\section{Experiments} \label{E}
\label{experiments}

\begin{figure*}[h]
    \centering
    \begin{subfigure}{0.66\columnwidth}
        \includegraphics[width=1.0\linewidth, height = 0.75\linewidth]{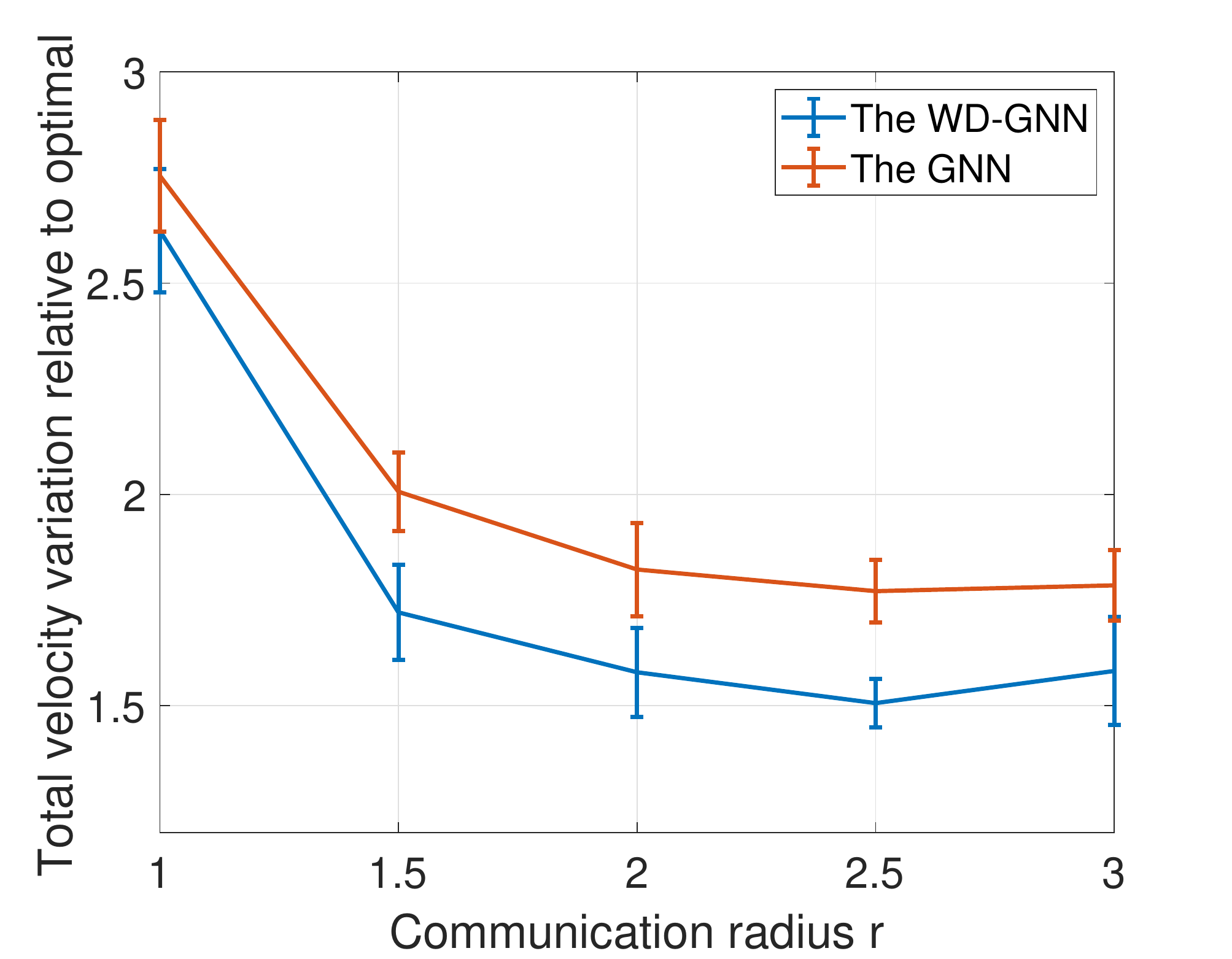}%
        \caption{}%
        \label{fig1}%
    \end{subfigure}\hfill\hfill%
    \begin{subfigure}{0.66\columnwidth}
        \includegraphics[width=1.0\linewidth,height = 0.75\linewidth]{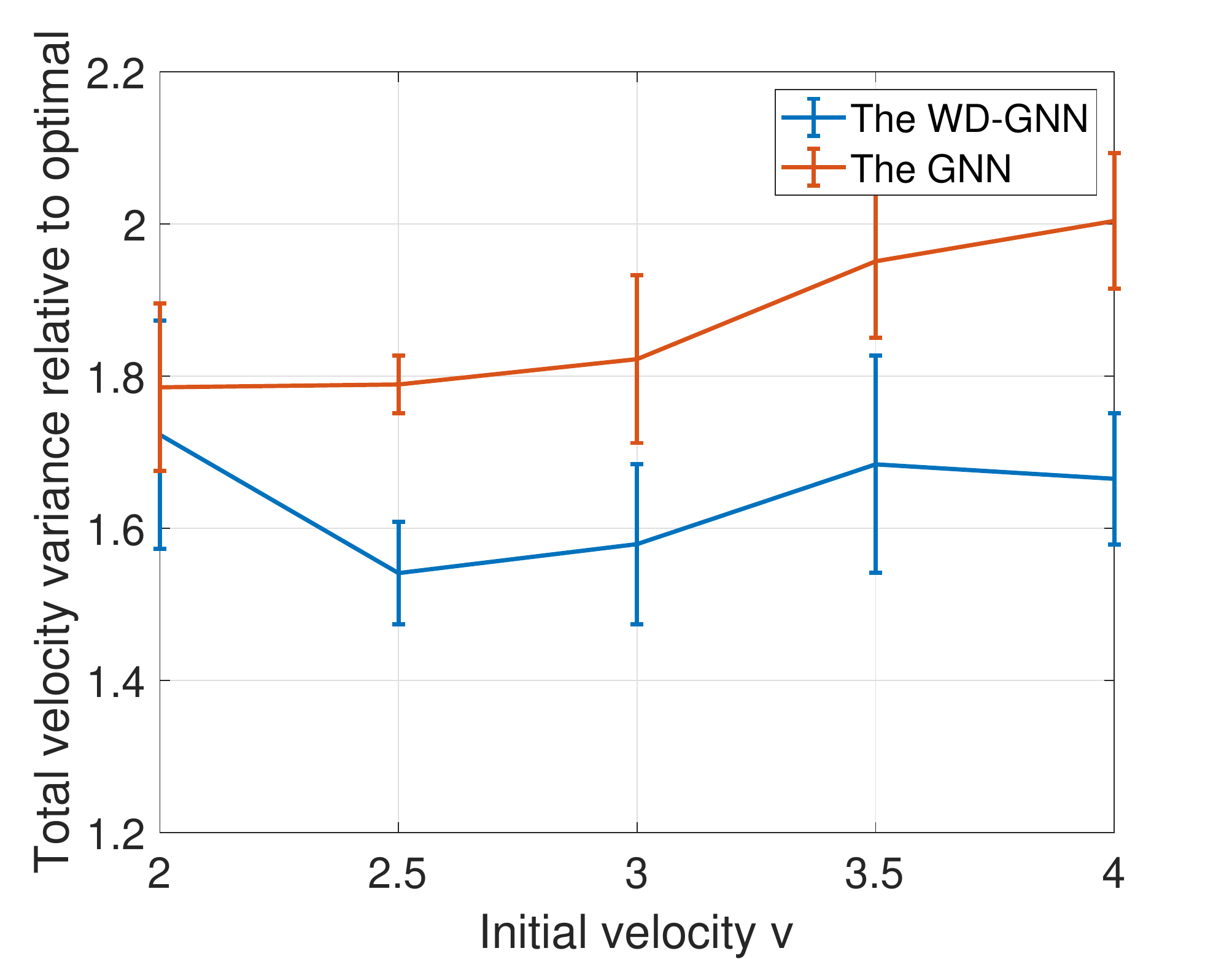}%
        \caption{}%
        \label{fig2}%
    \end{subfigure}\hfill\hfill%
    \begin{subfigure}{0.66\columnwidth}
        \includegraphics[width=1.0\linewidth,height = 0.75\linewidth]{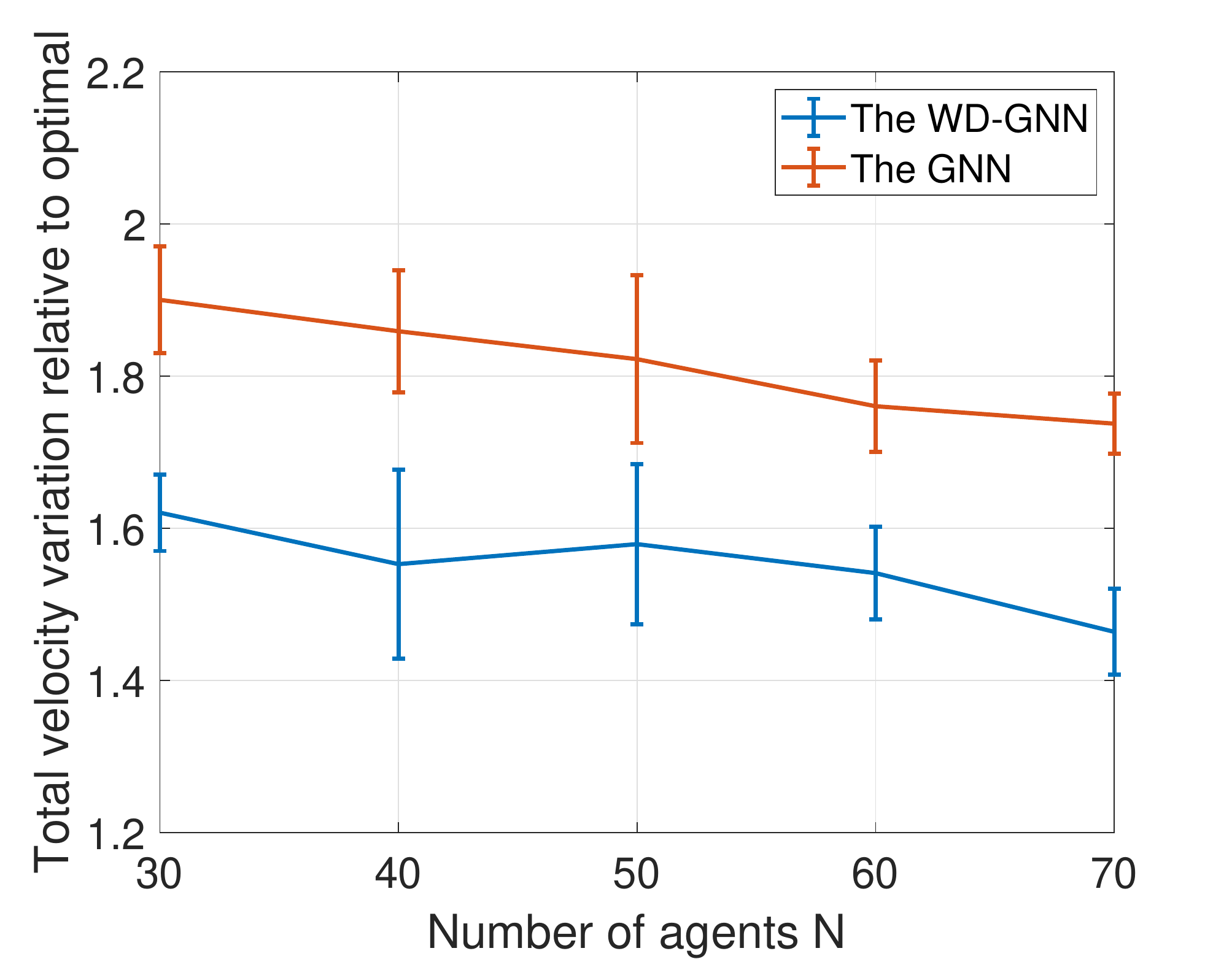}%
        \caption{}%
        \label{fig3}%
    \end{subfigure}%
    \caption{Total velocity variation relative to the optimal controller for the WD-GNN and the GNN. (a) Comparison under different communication radius. (b) Comparison under different initial velocities. (c) Comparison under different numbers of agents. }\label{fig:vary_n1}
\end{figure*}

\renewcommand\arraystretch{1.15}
\begin{table*}[ht]
    \begin{center}
        \caption{Average (std. deviation) of total and final velocity variations.}
        \label{table1}
        \begin{tabular}{|l|c|c|}
            \hline
            Architecture$/$Measurement & Total velocity variation & Final velocity variation \\ \hline
            Optimal controller & $52 (\pm 2)$ & $0.0035 (\pm 0.0001)$ \\ \hline
            WD-GNN &  $84 (\pm 5)$ & $0.0119 (\pm 0.0032)$ \\  \hline
            WD-GNN (centralized online learning) & $79(\pm 4)$ & $0.0065 (\pm 0.0028)$ \\  \hline
            WD-GNN (decentralized online learning) & $82(\pm 4)$ & $0.0069 (\pm 0.0023)$ \\ \hline
            GNN & $95(\pm 6)$ & $0.0153 (\pm 0.0030)$ \\ \hline
            Graph filter & $428 (\pm 105)$ & $ 1.8 (\pm 1.1)$ \\
            \hline
        \end{tabular}
    \end{center}
\end{table*}

The goal of the experiment is to learn a decentralized controller that coordinates robots to move together at the same velocity while avoiding collisions. Consider a network of $N$ robots initially moving at random velocities sampled in the interval $[-v, v]^2$. At time $t$, each robot $n_i$ is described by its position $\bbp_{i,t} \in \mathbb{R}^2$, velocity $\bbv_{i,t} \in \mathbb{R}^2$ and acceleration $\bbu_{i,t} \in \mathbb{R}^2$, the latter being the controllable variable. This problem has an optimal centralized controller $\bbu_{i,t}^{\ast}$ that can be readily computed \cite{tolstaya2019learning, Witsenhausen68-Counterexample}. Such a solution requires the knowledge of positions and velocities of all robots and thus demands a centralized computation unit. In the decentralized setting, we assume robot $n_i$ can only communicate with robot $n_j$ if they are within the communication radius $r$, i.e., $\| \bbp_{i,t}-\bbp_{j,t} \| \le r$. We establish the communication graph $\ccalG_t$ with the node set $\ccalV$ for robots and the edge set $\ccalE_t$ for available communication links, and the support matrix $\bbS_t$ is the associated adjacency matrix.

We use imitation learning \cite{Ross10-ImitationLearning} to train the WD-GNN for a decentralized controller on accelerations $\bbU_t=[\bbu_{1,t}, \ldots, \bbu_{N,t}]^\top=\bbPsi(\bbX_t; \bbS_t, \ccalA, \ccalB)\in \mathbb{R}^{N \times 2}$, where $\bbX_t$ is the graph signal that collects positions and velocities of neighbors at each robot \cite{tolstaya2019learning}. At testing time, we measure performance as the velocity variation among robots over the whole trajectory and also at the final time instant \cite{tolstaya2019learning}. In a way, the total velocity variation includes how long it takes for the robot swarm to be coordinated, while the final velocity variation tells how well the task was finally accomplished. For the WD-GNN, we consider the wide component as a graph filter and the deep component as a single layer GNN, where both have $G = 32$ output features. All filters are of order $K=3$ and the nonlinearity is the Tanh. The dataset contains $400$ trajectories for training, $40$ for validation and $40$ for testing. The ADAM optimizer is used with a learning rate $\gamma = 5 \cdot 10^{-4}$ and decaying factors $0.9$ and $0.999$. Our results are averaged for $5$ dataset realizations\footnote{See Appendix \ref{Appendix:implementationDetails} for comprehensive implementation details}.

We first compare the optimal controller, the GNN and the graph filter with the WD-GNN without online learning in the initial condition $r=2$m, $v=3$m/s and $N=50$ (Table \ref{table1}). We see that the WD-GNN exhibits the best performance in both performance measures. We attribute this behavior to the increased representation power of the WD-GNN as a combined architecture. The GNN takes the second place, while the graph filter performs much worse than the other two architectures. This is because the optimal distributed controller is known to be nonlinear \cite{Witsenhausen68-Counterexample}. 

To account for the adaptability to different initial conditions of the proposed model in comparison to the GNN, we further run simulations for changing communication radius $r$ (Fig. \ref{fig1}), changing initial velocities $v$ (Fig. \ref{fig2}) and changing numbers of agents $N$ (Fig. \ref{fig3}). We display results as the total velocity variation relative to the optimal controller. In general, these experiments show an improved robustness of the WD-GNN. Fig. \ref{fig1} shows that the performance increases as the communication radius $r$ increases, which is expected since the robots now have access to farther information. In Fig. \ref{fig2}, we observe that the flocking problem becomes harder with the increase of initial velocity $v$, since it is reasonably harder to control robots that move very fast in random directions. As for the number of agents $N$, Fig. \ref{fig3} demonstrates that the relative total velocity variation decreases when the number of agents $N$ increases, which is explained by the fact that the optimization problem becomes easier to solve with increased node exchanges in larger graphs.

Finally, we test the performance improvement on the WD-GNN with online learning. We consider the centralized online learning algorithm as the baseline, with the velocity variation among all robots as the instantaneous loss function in \eqref{eq:timevaryingp}. We test the proposed distributed online learning algorithm \eqref{eq:disgd}, where the instantaneous loss function is the velocity variance among neighboring robots. Results are shown in Table \ref{table1}. We see that the total and final velocity variations get reduced for both centralized and decentralized online algorithms, indicating that the WD-GNN is successfully adapting to the changing communication network. The improvement in the final velocity variation is more noticeable, since the effects of single time-updates get compounded as the trajectory proceeds.


\section{Conclusion}\label{C}

We put forward the Wide and Deep Graph Neural Network architecture and considered a distributed online learning scenario. The proposed architecture consists of a bank of filters (wide part) and a GNN (deep part), leading to a local and distributed architecture. To address time-varying problem scenarios without compromising the distributed nature of the architecture, we proposed a distributed online learning algorithm. By fixing the deep part, and retraining online only the wide part, we manage to obtain a convex optimization problem, and thus proved convergence guarantees of the online learning procedure. Numerical experiments are performed on learning decentralized controllers for flocking a robot swarm, showing the success of WD-GNNs in adapting to time-varying scenarios. Future research involves the online retraining of the deep part, as well as obtaining convergence guarantees for the distributed online learning algorithm.


\newpage
\clearpage
\appendix

\section{Proof of Theorem 1} \label{Appendix:proofOfTheorem1}

\begin{proof}

Let $\ccalA^{\dag}$ and $\ccalB^{\dag}$ be the parameters learned from the offline training phase. The proposed online learning procedure fixes the deep part, i.e., it freezes the parameters $\ccalA = \ccalA^{\dag}$, and retrains the wide part online. The model $\bbPsi(\bbX; \bbS, \ccalA, \ccalB)$ can then be represented as a function of the wide part parameters $\ccalB$ only
 \begin{equation}\label{proof:theorem1eq1}
 \begin{aligned}
\bbPsi(\bbX; \bbS, \ccalA^{\dag}, \ccalB) &= \alpha_{\text{D}} \bbPhi(\bbX; \bbS, \ccalA^{\dag}) + \alpha_{\text{W}} \bbB(\bbX; \bbS, \ccalB) + \beta\\
& = \hat{\bbPsi}(\bbX; \bbS, \ccalB).
\end{aligned}
\end{equation}
Given the graph signal $\bbX$ and the graph matrix $\bbS$, $\hat{\bbPsi}(\bbX; \bbS, \ccalB)$ is a linear function of $\ccalB$ since both the graph filter $\bbB(\bbX; \bbS, \ccalB)$ and the combination way of two components are linear.

At testing time $t$, the sampled optimization problem \eqref{eq:timevaryingp} is translated to
 \begin{equation}\label{eq:timevaryingp1}
\min_{\ccalB} J_t\big(\hat{\bbPsi}(\bbX_t; \bbS_t, \ccalB)\big)
\end{equation}
where $J_t(\cdot)$, $\bbX_t$ and $\bbS_t$ are given instantaneous loss function, observed signal and graph matrix at time $t$. Since $\hat{\bbPsi}(\bbX_t; \bbS_t, \ccalB)$ is a linear function of $\ccalB$, $J_t\big(\hat{\bbPsi}(\bbX_t; \bbS_t, \ccalB)\big)$ is differentiable, strongly smooth with constant $C_{t,s}$ and strongly convex with constant $C_{t,\ell}$ based on Assumption 2. This implies that the problem \eqref{eq:timevaryingp1} is a time-varying convex optimization problem. We can then prove the theorem by following the proof of Corollary 7.1 in \cite{simonetto2017time}.
\end{proof}

\section{Implementation details for robot swarm control in flocking} \label{Appendix:implementationDetails}

We consider a network with $N$ robots initially moving at random velocities. At time $t$, each robot $n_i$ is described by its position $\bbp_{i,t} \in \mathbb{R}^2$, velocity $\bbv_{i,t} \in \mathbb{R}^2$, and controls its acceleration $\bbu_{i,t} \in \mathbb{R}^2$ for the next state
\begin{equation}\label{eq:controlProcess}
\bbp_{i,t+1} = \bbp_{i,t} + \bbv_{i,t} T_s + \frac{1}{2}\bbu_{i,t}T_s^2, ~ \bbv_{i,t+1} = \bbv_{i,t}+\bbu_{i,t}T_s
\end{equation}
where $T_s$ is the sampling time and $\bbu_{i,t}$ is held constant during the sampling time interval $[T_s t,T_s (t+1)]$. Our purpose is to control accelerations $\bbU_t=[\bbu_{i,t}, \ldots, \bbu_{N,t}]^\top \in \mathbb{R}^{N \times 2}$ such that robots will ultimately move at the same velocity without collision. There is an optimal solution for accelerations \cite{Witsenhausen68-Counterexample}
\begin{equation}\label{eq:optimalController}
\bbu_{i,t}^* = - \sum_{j=1}^N \left(\bbv_{i,t} - \bbv_{j,t} \right) - \sum_{j=1}^N \nabla_{\bbp_{i,t}} V\left( \bbp_{i,t}, \bbp_{j,t} \right)
\end{equation}
for all $i=1,\ldots,N$ with
\begin{align} \label{eq:avoidancePotential}
    &V( \bbp_{i,t}, \bbp_{j,t}) \\
    &=\begin{cases}
        \frac{1}{\| \bbp_{i,t}-\bbp_{j,t} \|^2} \!-\! \log \left( \| \bbp_{i,t}\!-\!\bbp_{j,t} \|^2 \right) ,& \text{if}~ \| \bbp_{i,t}-\bbp_{j,t} \| \le \rho\\\\
        \frac{1}{\rho^2} - \log \left( \rho^2 \right), & \text{otherwise}
    \end{cases}\nonumber
    \end{align}
the collision avoidance potential. The computation of $\bbu_{i,t}^*$ requires instantaneous positions $\{ \bbp_{i,t} \}_{i=1}^N$ and velocities $\{ \bbv_{i,t} \}_{i=1}^N$ of all robots over network. As such, it is centralized and cannot be implemented in practice, where each robot only has access to local neighborhood information.

In the decentralized setting, robot $n_i$ can communicate with robot $n_j$ if and only if they are within the communication radius $r$, i.e., there is a communivation link $(n_i,n_j)$ if $\| \bbp_{i,t}-\bbp_{j,t} \| \le r$. We establish the communication graph $\ccalG_t$ with the node set $\ccalV = \{ n_1,\ldots,n_N \}$ and the edge set $\ccalE_t$ containing available communication links. The graph matrix $\bbS_t$ is the adjacency matrix with entry $[\bbS_t]_{ij} =1$ if $(n_i,n_j) \in \ccalE_t$ and $[\bbS_t]_{ij} =0$ otherwise. Additionally, we assume robot communications occur within the sampling time interval, such that robot action clock and communication clock coincide.

We use the WD-GNN to learn a decentralized controller $ \bbU_t = \bbPsi(\bbX_t; \bbS_t, \ccalA, \ccalB)$ where the graph matrix $\bbS_t$ is the adjacency matrix of the communication graph $\ccalG_t$ and the graph signal $\bbX_t=[ \bbx_{1,t}, \ldots, \bbx_{N,t}]^\top \in \mathbb{R}^{N \times 6}$ is
\begin{equation}\label{eq:relevantFeature}
\begin{aligned}
\bbx_{i,t} \!=\!\! &\left[\sum_{j:n_j\in \ccalN_{i,t}}\!\!\!\!\!\!\big(\bbv_{i,t} \!-\! \bbv_{j,t} \big), \sum_{j:n_j\in \ccalN_{i,t}}\frac{\bbp_{i,t}-\bbp_{j,t}}{\| \bbp_{i,t}-\bbp_{j,t} \|^4}, \right.\\
& \left. \quad \quad \quad \quad \quad \quad \quad \quad \quad  \sum_{j:n_j\in \ccalN_{i,t}}\frac{\bbp_{i,t}-\bbp_{j,t}}{\| \bbp_{i,t}-\bbp_{j,t} \|^2}\right]
\end{aligned}
\end{equation}
for all $i=1,\ldots,N$, which is a local feature vector collecting position and velocity information of neighboring robots. Graph filters here are adapted to the delayed information structure as
\begin{equation} \label{eqn:delayedGraphFilter}
    \bbB(\bbX_t; \bbS_t, \ccalB) = \sum_{k = 0}^{K} \bbS_t \bbS_{t-1} \cdots \bbS_{t-k} \bbX_{t-k} \bbB_{k}.
\end{equation}
We leverage the imitation learning \cite{Ross10-ImitationLearning} to parametrize the optimal controller \eqref{eq:optimalController} with the WD-GNN.

$\textbf{Dataset.}$ The dataset contains $400$ trajectories for training, $40$ for validation and $40$ for testing. We generate each trajectory by initially positioning $N=50$ robots randomly in a circle. The minimal initial distance between two robots is $0.1$m and initial velocities are sampled randomly from $[-v, +v]^2$ with $v=3$m/s by default. The duration of trajectories is $T=2$s with the sampling time $T_s=0.01$s; the maximal acceleration is $\pm 10 m/{s^2}$ and the communication radius is $r=2$m.

$\textbf{Parametrizations.}$ For the WD-GNN, we consider the wide component as a graph filter and the deep component as a single layer GNN, where both have $G=32$ output features. All filters are of order $K=3$ and the nonlinearity is the Tanh. The output features are fed into a local readout layer to generate two-dimensional acceleration $\bbu_{i,t}$ at each node. In the offline training phase, the combination parameters $\alpha_D$m $\alpha_W$ and $\beta$ are considered as architecture parameters and are trained jointly. We train the WD-GNN for $30$ epochs with batch size of $20$ trajectories. The ADAM optimizer is used with decaying factors $\beta_1=0.9$, $\beta_2 = 0.999$ and learning rate $\gamma = 5 \cdot 10^{-4}$. We average experiment results for $5$ dataset realizations.

$\textbf{Performance measure.}$ The flocking condition can be quantified by the variance of robot velocities over network, referred as the \emph{velocity variation}. At testing time, we measure the performance of learned controller from two aspects: the total velocity variation over the whole trajectory
\begin{align}\label{eq:totalVariance}
&\ccalJ = \frac{1}{N} \sum_{t=1}^D \sum_{i=1}^N \left\| \bbv_{i,t} - \frac{1}{N} \sum_{i=1}^N \bbv_{i,t} \right\|^2\!
\end{align}
and
\begin{align}\label{eq:finalVariance}
& J(D) = \frac{1}{N} \sum_{i=1}^N \left\| \bbv_{i,D} - \frac{1}{N} \sum_{i=1}^N \bbv_{i,D} \right\|^2
\end{align}
where $D=T/T_s$ is the total number of time instances. The former reflects the whole controlling process which decreases if robots approach the consensus more quickly, while the latter indicates the final flocking condition of robots.

Main results are shown in the full paper. To further help understand and visualize this experiment, we show video snapshots of the robot swarm flocking process using the learned WD-GNN decentralized controller in Fig. \ref{fig.video}. We see that robots move at random velocities initially in Fig. \ref{subfig1}, tend to move together in \ref{subfig2}, and are well coordinated in Fig. \ref{subfig3}. The cost shown in the figure is the instantaneous velocity variation of robots over network.

\begin{figure*}[t]
\begin{subfigure}{0.66\columnwidth}
\includegraphics [width=1.1\linewidth, height = 0.85\linewidth]
                 {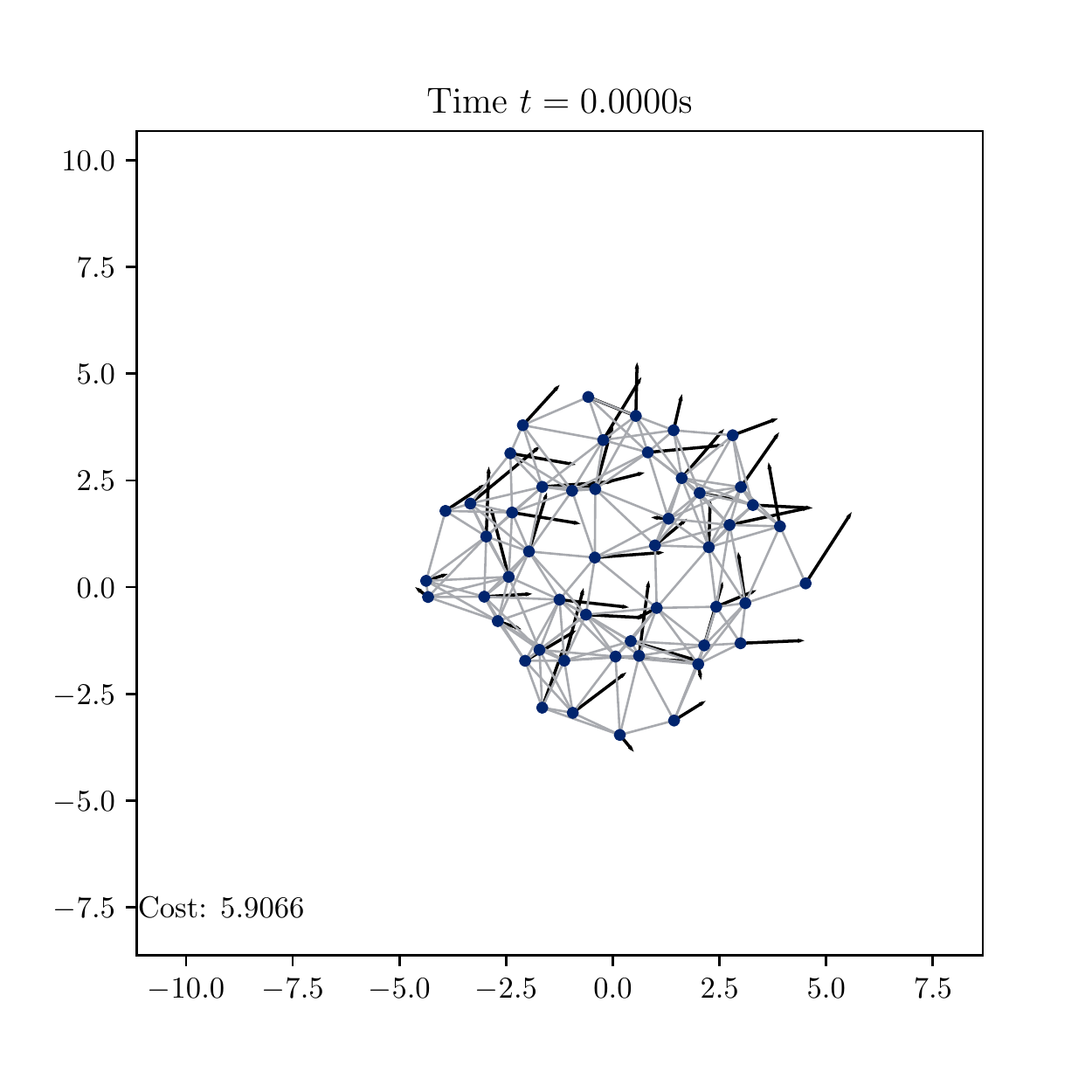}\qquad
\caption{Time $t=0.00$s}%
\label{subfig1}%
\end{subfigure}\hfill\hfill%
\begin{subfigure}{0.66\columnwidth}
\includegraphics [width=1.1\linewidth, height = 0.85\linewidth]
                 {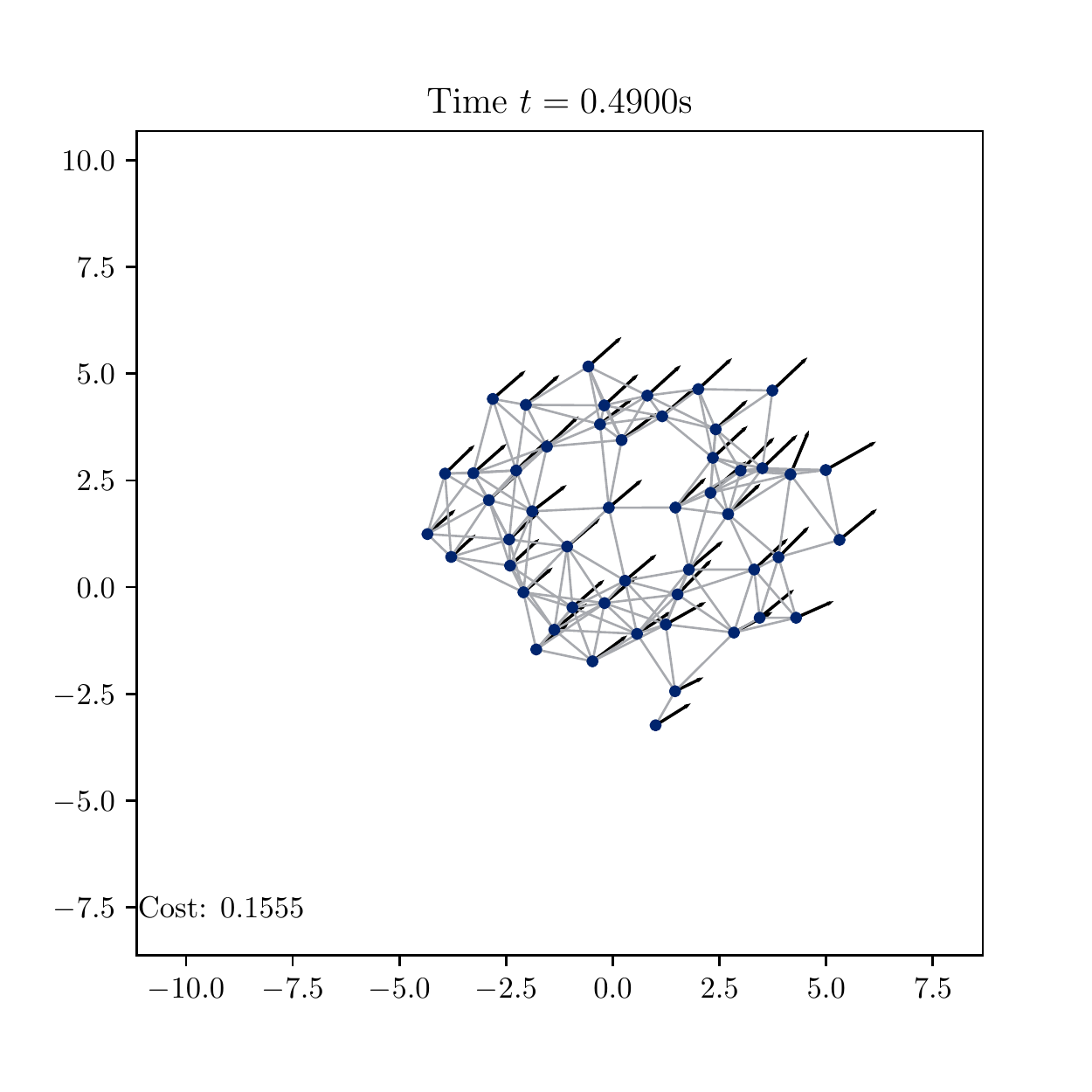}\qquad
\caption{Time $t=0.49$s}%
\label{subfig2}%
\end{subfigure}\hfill\hfill%
\begin{subfigure}{0.66\columnwidth}
\includegraphics [width=1.1\linewidth, height = 0.85\linewidth]
                 {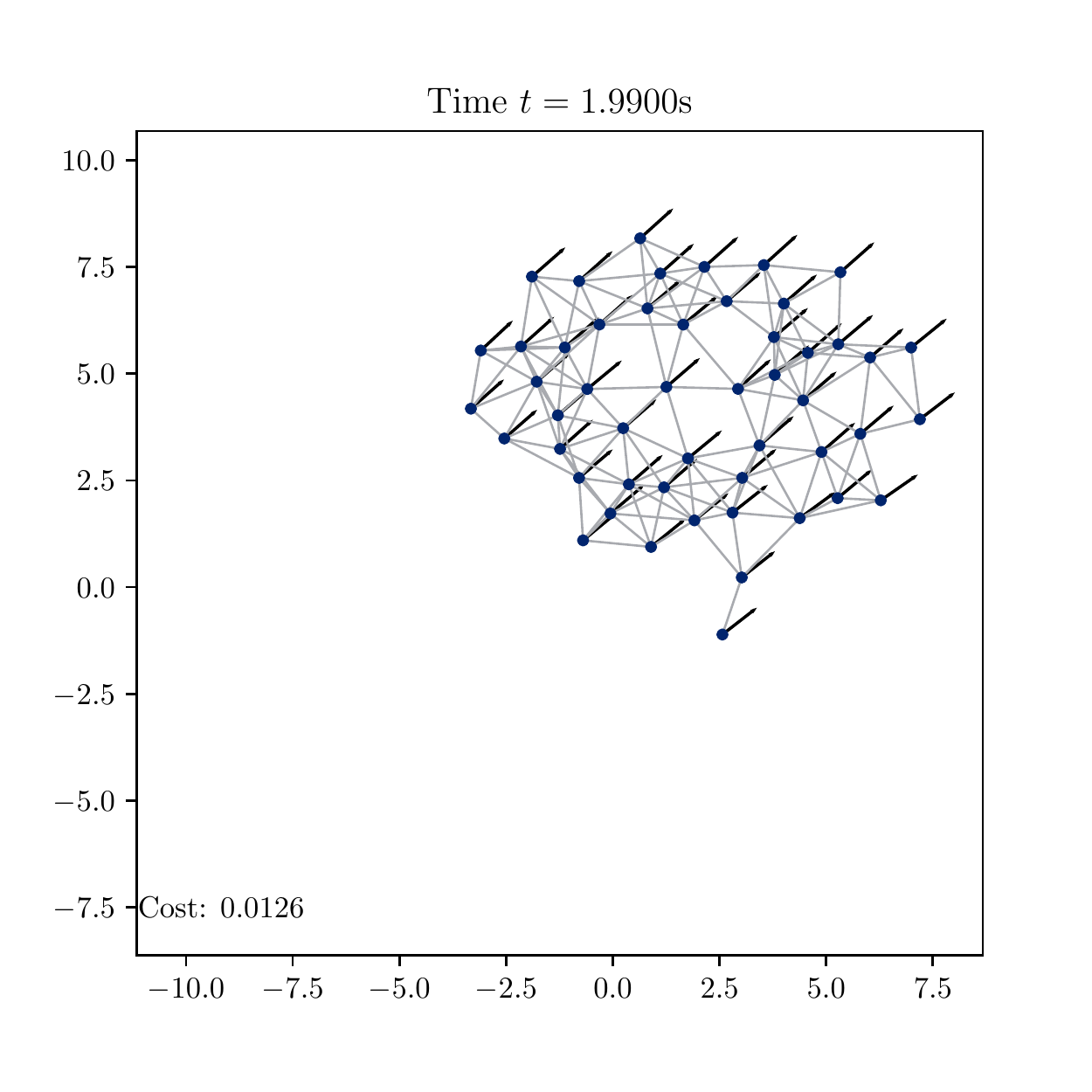}\qquad
\caption{Time $t=1.99$s}%
\label{subfig3}%
\end{subfigure}
\caption{Video snapshots of the robot swarm flocking process with the learned WD-GNN controller.}
\label{fig.video}
\end{figure*}

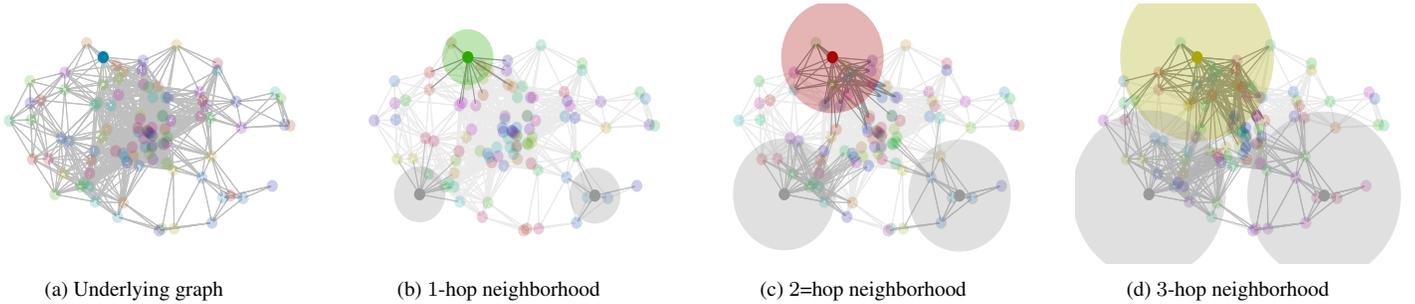
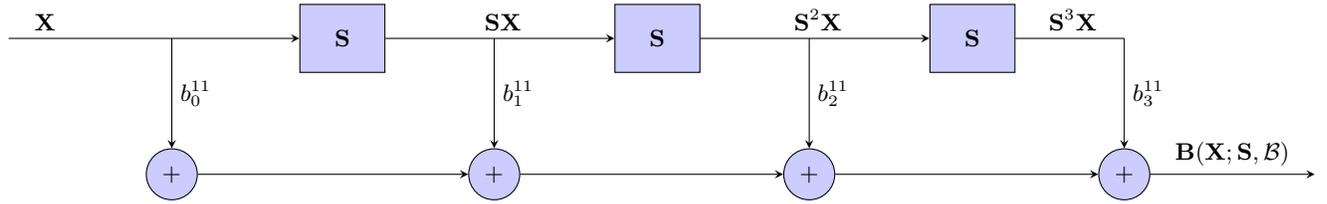
\begin{figure*}[t]
\begin{subfigure}{0.45\columnwidth}
\includegraphics [width=1.1\linewidth, height = 0.85\linewidth]
                 {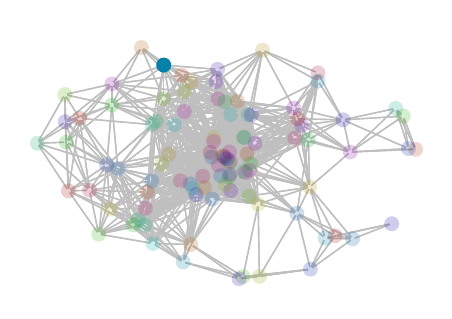}\qquad
\caption{Underlying graph}%
\label{subfiga_vary_p}%
\end{subfigure}\hfill\hfill%
\begin{subfigure}{0.45\columnwidth}
\includegraphics [width=1.1\linewidth, height = 0.85\linewidth]
                 {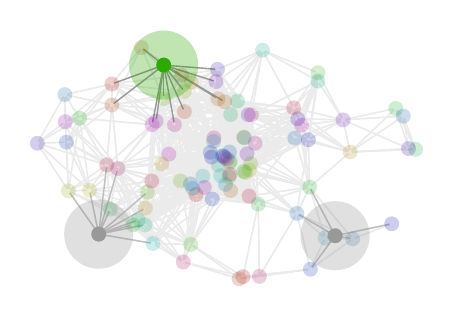}\qquad
\caption{$1$-hop neighborhood}%
\label{subfiga_vary_p}%
\end{subfigure}\hfill\hfill%
\begin{subfigure}{0.45\columnwidth}
\includegraphics [width=1.1\linewidth, height = 0.85\linewidth]
                 {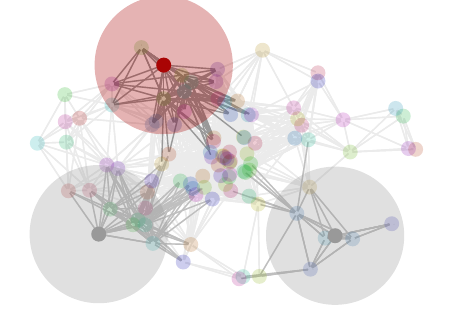}\qquad
\caption{$2$=hop neighborhood}%
\label{subfiga_vary_p}%
\end{subfigure}\hfill\hfill%
\begin{subfigure}{0.45\columnwidth}
\includegraphics [width=1.1\linewidth, height = 0.85\linewidth]
                 {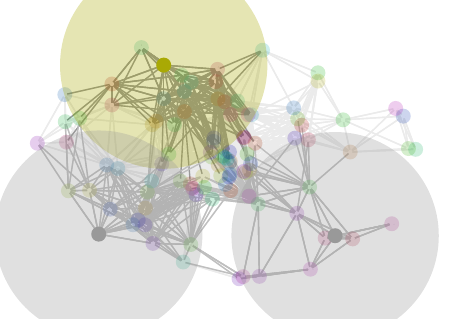} 
\caption{$3$-hop neighborhood}%
\label{subfiga_vary_p}%
\end{subfigure}\\ \bigskip
\def \thisplotscale {1.8}
\def \unit {\thisplotscale cm}

\tikzstyle {Phi} = [rectangle,
                    thin,
                    minimum width = 1.0*\unit,
                    minimum height = \sumshift*\unit,
                    anchor = west,
                    draw,
                    fill = blue!20]

\tikzstyle {sum} = [circle,
                    thin,
                    minimum width  = 0.3*\unit,
                    minimum height = 0.3*\unit,
                    anchor = center,
                    draw,
                    fill = blue!20]

\def \deltax {1.5}
\def \deltay {0.8}
\def \sumshift {0.4}

\def \thisplotscale {2.26}
\def \unit {\thisplotscale cm}

\tikzstyle {Phi} = [rectangle,
                    thin,
                    minimum width = 0.5*\unit,
                    minimum height = \sumshift*\unit,
                    anchor = west,
                    draw,
                    fill = blue!20]

\tikzstyle {sum} = [circle,
                    thin,
                    minimum width  = 0.3*\unit,
                    minimum height = 0.3*\unit,
                    anchor = center,
                    draw,
                    fill = blue!20]

\def \deltax {1.6}
\def \deltay {0.8}
\def \sumshift {0.4}

\begin{tikzpicture}[x = 1*\unit, y = 1*\unit]

\node (origin) [] {};
\path (origin) ++ (0.2*\deltax, 0) node (first) [] {};

\path (first) ++ (1.1*\deltax, 0) node (0) [Phi] {$\bbS$};
\path (0)     ++ (1.0*\deltax, 0) node (1) [Phi] {$\bbS$};
\path (1)     ++ (1.0*\deltax, 0) node (2) [Phi] {$\bbS$};

\path (2.east) ++ (0.7*\sumshift*\deltax, 0) node [anchor=west] (last) [] {};

\path (first.east) ++ (1.5*\sumshift*\deltax, -\deltay) node (sum0) [sum] {$+$};
\path (0.east) ++ (\sumshift*\deltax, -\deltay) node (sum1) [sum] {$+$};
\path (1.east) ++ (\sumshift*\deltax, -\deltay) node (sum2) [sum] {$+$};
\path (2.east) ++ (\sumshift*\deltax, -\deltay) node (sum3) [sum] {$+$};

\path[-stealth] (first) edge [very near start, above] node {$\bbX$}               (0);	
\path[-stealth] (0)     edge [above] node {$\ \bbS\bbX$}     (1);	
\path[-stealth] (1)     edge [above] node {$\ \bbS^{2}\bbX$} (2);	
\path[-]        (2)     edge [above] node {$\ \bbS^{3}\bbX$} (sum3 |- last);

\path[-stealth, draw] (sum0 |- first) -- (sum0) node [midway, right] {$b^{11}_{0}$};	
\path[-stealth, draw] (sum1 |- 0)     -- (sum1) node [midway, right] {$b^{11}_{1}$};	
\path[-stealth, draw] (sum2 |- 1)     -- (sum2) node [midway, right] {$b^{11}_{2}$};	
\path[-stealth, draw] (sum3 |- 2)     -- (sum3) node [midway, right] {$b^{11}_{3}$};

\path[-stealth, draw] (sum0) -- (sum1);	
\path[-stealth, draw] (sum1) -- (sum2);	
\path[-stealth, draw] (sum2) -- (sum3);	

\path[-stealth] (sum3) edge [above] node
                {$\bbB(\bbX;\bbS,\ccalB)$} ++ (0.7*\deltax, 0);

\end{tikzpicture}
\caption{Graph filters perform successive local node exchanges with neighbors, where the $k$-shifted signal $\bbS^k \bbX$ collects the information from $k$-hop neighborhood (shown by the increasing disks), and aggregate these shifted signals $\bbX, \ldots, \bbS^K \bbX$ with a set of parameters $[b^{11}_0, \ldots, b^{11}_K]^\top$ to generate the higher-level feature that accounts for the graph structure up to a neighborhood of radius $K$.}
\label{fig.graphcon}
\end{figure*}

\section{Experiments on movie recommendation systems}

We consider another experiment on movie recommendation systems to further corroborate our model. The goal is to predicate the rating a user would give to a specific movie \cite{Huang2018}. We build the underlying graph as the movie similarity network, where nodes are movies and edge weights are similarity strength between movies. The graph signal contains the ratings of movies given by a user, with missing values if those movies are not rated. We train the WD-GNN to predict the rating of a movie of our choice, based on ratings given to other movies.

\subsection{Implementation Details}

We use a subset of MovieLens 100k dataset, which includes $943$ users and $400$ movies with largest number of ratings \cite{Harper2016}. We compute the movie similarity as the Pearson correlation and keep ten edges with highest similarity for each node (movie) \cite{Huang2018}. Each user is a graph signal, where the signal value on each node is the rating of its associated movie given by this user, with zero value if that movie is not rated. The dataset is split into $90 \%$ for training and $10\%$ for testing. The rating of the movie of our choice is extracted as a label at the training phase, and zeroed out in the graph signal at the testing phase. 

We consider a WD-GNN comprising a graph filter and a single layer GNN. Both components have $G=64$ output features. All filters are of order $K=5$ and the nonlinearity is the ReLU. A local readout layer follows to map $64$ output features to a single scalar predicted rating at each node. We train the WD-GNN for $30$ epochs with batch size of $5$ samples, and use the ADAM optimizer with decaying factors $\beta_1=0.9$, $\beta_2 = 0.999$ and learning rate $\gamma = 5 \cdot 10^{-3}$. The performance is measured with the root mean squared error (RMSE) and the results are averaged for $10$ random dataset split.

\subsection{Results}

\begin{table*}
\begin{center}
\caption{Average (std. deviation) of the root mean squared error.}
\label{table3}
\begin{tabular}{|l|c|c|}
\hline
Architecture$/$Experiments & Train $\&$ test on same movie & Train on one movie $\&$ test on another movie \\ \hline
WD-GNN &  $0.8535 (\pm 0.0883)$ & $1.0920 (\pm 0.1053)$ \\  \hline
Online WD-GNN ($1$ gradient descent step) & $0.8524(\pm 0.1007)$ & $0.9759 (\pm 0.0927)$ \\  \hline
Online WD-GNN ($2$ gradient descent steps) & $0.8442(\pm 0.0993)$ & $0.9815 (\pm 0.0902)$ \\  \hline
Online WD-GNN ($3$ gradient descent steps) & $0.8585(\pm 0.1063)$ & $0.9909 (\pm 0.0901)$ \\  \hline
GNN & $0.8630(\pm 0.0884)$ & $1.0889 (\pm 0.1106)$ \\ \hline
Graph filter & $0.8589 (\pm 0.0895)$ & $ 1.0950 (\pm 0.1129)$ \\
\hline
\end{tabular}
\end{center}  \vspace{-4mm}
\end{table*}

We first consider the WD-GNN without online retraining and compare the proposed architecture with the GNN and the graph filter (Table \ref{table3}). We predict ratings for the movie that has the largest number of ratings in the dataset, \emph{Star War}. We see that though three architectures exhibit comparable performance, the WD-GNN performs best with the lowest RMSE. We also attribute this behavior to the increased learning ability of the WD-GNN obtained from its combined architecture. The second experiment considers the transferability, where we train the architectures on one movie (\emph{Star War}) and use the learned models to predict ratings for another movie (\emph{Contact}). In this case, the problem scenario changes that creates a mismatch between the training phase and the testing phase. This change, as expected in theoretical findings, degrades the performance of all three architectures severely as shown in Table \ref{table3}. In this case, the proposed online learning procedure for the WD-GNN becomes important to improve the performance, which is demonstrated next. 

We then run the WD-GNN with online learning. At testing time, we consider the recommendation system gets feedback from the user after it predicted the rating, which is used as the label to compute the instantaneous loss function in \eqref{eq:timevaryingp} for online learning. While the centralized online learning is available for recommendation systems, we keep in mind that the proposed WD-GNN can be retrained online in a distributed manner. We consider an online procedure experiencing $400$ testing users and for each user, the system performs $1$, $2$ or $3$ gradient descent steps to retrain the wide part based on the instantaneous signal and the feedback label, respectively. Results are shown in Table \ref{table3}. We observe significant performance improvements for the online WD-GNN when training and testing on different movies. As the testing scenario differs from the one used for training, online learning adapts the WD-GNN to the new scenario and thus improves the transferability. We remark that these improvements will be emphasized as the testing phase further goes on with more users involved. On the other hand, when training and testing on the same movie, online learning exhibits slight improvements on the performance with slightly lower RMSE. This is because the problem scenario does not change much and the offline phase has already trained the WD-GNN well.

\section{Online Wide and Deep GNN evaluation} 

In Fig. \ref{fig.graphcon}, we show details about the graph filter (graph convolution) [(3) in full paper] of order $K=3$ with $F=1$ input feature and $G=1$ output feature, i.e.,
\begin{equation} \label{eqn:graphConv11}
    \bbB(\bbX; \bbS, \ccalB) = \sum_{k = 0}^{K} b^{11}_{k} \bbS^{k} \bbX
\end{equation}
where filter parameters are $\ccalB = \{ b^{11}_0,\ldots,b^{11}_K \}$. In particular, the linear operation $\bbS \bbX$, also referred to as graph shift operation, leverages the graph structure to process the graph signal. It assigns to each node the aggregated signal from immediate neighbors and collects the graph neighborhood information. Shifting $\bbX$ for $k$ times aggregates information from $k$-hop neighborhood yielding the $k$-shifted signal $\bbS^k \bbX$. With a set of parameters $[b^{11}_0, \ldots, b^{11}_K]^\top \in \mathbb{R}^{K+1}$, the graph filter generates the higher-level feature that accounts for shifted signals up to a neighborhood of radius $K$, and thus reflects a more complete picture of network. With the shift-and-sum operation of graph signal $\bbX$ over graph structure $\bbS$, the graph filter is also considered as the convolution in graph domain. If further particularizing $\bbS$ the line graph and $\bbx$ the signal sampled at time instances, the graph filter \eqref{eqn:graphConv11} reduces to the conventional convolution.

Furthermore, we summarize the proposed online learning algorithm for the WD-GNN in Algorithm \ref{alg:OLWD} for a clear understanding.
 
 {\linespread{1.5}
\begin{algorithm}[H] \begin{algorithmic}[1]
\State \textbf{Input:} offline learned parameters $\ccalA^{\dag}$, $\ccalB^{\dag}$ by minimizing the ERM problem [cf. \eqref{eqn:ERM} in full paper] over the training dataset, and online step size $\gamma_t$
\State Fix the deep part parameters $\ccalA = \ccalA^{\dag}$ and set the initial wide part parameters $\ccalB_0 = \ccalB^{\dag}$
\For {$t = 0,1,2...$}
      \State Observe instantaneous graph signal $\bbX_t$, graph matrix $\bbS_t$ and loss function $J_t(\cdot)$
      \State Compute the instantaneous loss $J_t\big(\bbPsi(\bbX_t; \bbS_t, \ccalA^{\dag}, \ccalB_t)\big)$
	  \If {~requiring decentralized implementation~}
	      	  \State Update the wide part parameters in a distributed manner
	      	  \For {$i = 1,...,N$}
	          \State $\ccalB_{i,t+1} = \frac{1}{N_i+1} \big( \sum_{j:n_j \in \ccalN_i} \ccalB_{j,t} + \ccalB_{i,t}\big) - \gamma_t \nabla_\ccalB J_{i,t}\big(\bbPsi(\bbX_t; \bbS_t, \ccalA^{\dag}, \ccalB_{i,t})\big)$
	          \EndFor
	  \Else 
	          \State Update the wide part parameters in a centralized manner
	          \State $\ccalB_{t+1} = \ccalB_{t} - \gamma_t \nabla_\ccalB J_t\big(\bbPsi(\bbX_t; \bbS_t, \ccalA^{\dag}, \ccalB_{t})\big)$
      \EndIf
\EndFor
\end{algorithmic}
\caption{Online Learning Algorithm for the WD-GNN}\label{alg:OLWD} \end{algorithm}}


\bibliographystyle{IEEEbib}
\bibliography{myIEEEabrv,biblioOp}

\end{document}

%% file: mySections.tex
\usepackage{needspace}



%% file: myTikzDefs.tex
\def\unit{\myfactor cm}

\pgfplotsset{ ylabel near ticks,                 
              xlabel near ticks,                 
              tick label style = {font=\footnotesize},   
              label style = {font=\footnotesize}, 
              title style = {font=\footnotesize},
            }

\tikzstyle{empty node} = [ circle, 
                     draw = black,
                     text = black, 
                     minimum size = 0.8*\unit]

\tikzstyle{node} = [ empty node, 
                     fill = blue!50,
                     draw = blue!50,
                     text = white]

\tikzstyle{blue node} = [ empty node, 
                         fill = blue!50,
                         draw = blue!50,
                         text = white]

\tikzstyle{red node} = [ empty node, 
                         fill = red!50,
                         draw = red!50,
                         text = white]

\tikzstyle{green node} = [ empty node, 
                         fill = mygreen!50,
                         draw = mygreen!50,
                         text = white]

\tikzstyle{black node} = [ empty node, 
                           fill = black!40,
                           draw = black!40,
                           text = white]

\tikzstyle{empty dot} = [ circle, 
                           draw = black,
                           inner sep = 0pt,
                           anchor = center,
                           minimum size = 0.4*\unit]

\tikzstyle{dot} = [ empty dot, 
                    fill = blue!50,
                    draw = blue!50]

\tikzstyle{blue dot} = [ empty dot, 
                         fill = blue!50,
                         draw = blue!50]

\tikzstyle{red dot} = [ empty dot, 
                        fill = red!50,
                        draw = red!50]

\tikzstyle{black dot} = [ empty dot, 
                          fill = black!50,
                          draw = black!50]

\tikzstyle{edge}                 = [shorten >=1pt, shorten <=1pt]
\tikzstyle{directed edge}        = [edge, -stealth]
\tikzstyle{double directed edge} = [edge, stealth-stealth]
\tikzstyle{tight edge}                 = [shorten >=0pt, shorten <=0pt]

\tikzstyle{block} = [ rectangle,
                      minimum width = \unit,
                      minimum height = \unit,
                      fill = blue!15,
                      draw = black,
                      text = black]